\documentclass[runningheads]{llncs}
\usepackage[T1]{fontenc}
\usepackage{graphicx}
\usepackage{listings}
\usepackage{epsfig}
\usepackage{graphicx}
\usepackage{amsmath}
\usepackage{amssymb}

\usepackage{algorithm}
\usepackage{algorithmic}

\usepackage{booktabs}
\usepackage{diagbox}
\usepackage{stackengine}
\usepackage{color}
\usepackage{xcolor}

\usepackage{threeparttable}
\usepackage{multirow}
\usepackage{wrapfig}
\usepackage{hyperref}

\usepackage{listings}
\usepackage{xcolor}
\definecolor{codegreen}{rgb}{0,0.6,0}
\definecolor{codegray}{rgb}{0.5,0.5,0.5}
\definecolor{codepurple}{rgb}{0.58,0,0.82}
\definecolor{backcolour}{rgb}{0.95,0.95,0.92}
\lstdefinestyle{mystyle}{
    backgroundcolor=\color{backcolour},   
    commentstyle=\color{codegreen},
    keywordstyle=\color{magenta},
    numberstyle=\tiny\color{codegray},
    stringstyle=\color{codepurple},
    basicstyle=\ttfamily\footnotesize,
    breakatwhitespace=false,         
    breaklines=true,                 
    captionpos=b,                    
    keepspaces=true,                 
    numbers=left,                    
    numbersep=5pt,                  
    showspaces=false,                
    showstringspaces=false,
    showtabs=false,                  
    tabsize=2
}
\begin{document}
%
\title{Joint Class-Affinity Loss Correction for Robust Medical Image Segmentation with Noisy Labels}

\author{Xiaoqing Guo\orcidID{0000-0002-9476-521X} \and Yixuan Yuan \orcidID{0000-0002-0853-6948}}
\authorrunning{X. Guo and Y. Yuan}
%
\institute{Department of Electrical Engineering, City Univeristy of Hong Kong, Hong Kong
\email{xiaoqiguo2-c@my.cityu.edu.hk, yxyuan.ee@cityu.edu.hk}}

\maketitle              
\begin{abstract}
Noisy labels collected with limited annotation cost prevent medical image segmentation algorithms from learning precise semantic correlations. Previous segmentation arts of learning with noisy labels merely perform a pixel-wise manner to preserve semantics, such as pixel-wise label correction, but neglect the pair-wise manner. In fact, we observe that the pair-wise manner capturing affinity relations between pixels can greatly reduce the label noise rate. Motivated by this observation, we present a novel perspective for noisy mitigation by incorporating both pixel-wise and pair-wise manners, where supervisions are derived from noisy class and affinity labels, respectively. Unifying the pixel-wise and pair-wise manners, we propose a robust Joint Class-Affinity Segmentation (JCAS) framework to combat label noise issues in medical image segmentation. Considering the affinity in pair-wise manner incorporates contextual dependencies, a differentiated affinity reasoning (DAR) module is devised to rectify the pixel-wise segmentation prediction by reasoning about intra-class and inter-class affinity relations. To further enhance the noise resistance, a class-affinity loss correction (CALC) strategy is designed to correct supervision signals via the modeled noise label distributions in class and affinity labels. Meanwhile, CALC strategy interacts the pixel-wise and pair-wise manners through the theoretically derived consistency regularization. Extensive experiments under both synthetic and real-world noisy labels corroborate the efficacy of the proposed JCAS framework with a minimum gap towards the upper bound performance. The source code is available at \url{https://github.com/CityU-AIM-Group/JCAS}.
\end{abstract}

\section{Introduction}
Image segmentation, as one of the most essential tasks in medical image analysis, has received lots of attention over the last decades. This task aims to assign a semantic label for each pixel, further benefiting various clinical applications such as treatment planning and surgical navigation \cite{karimi2021convolution}. Deep learning algorithms based on convolutional neural networks (CNNs) have achieved remarkable progress in medical image segmentation, but they require a large amount of training data with precise pixel-level annotations that are extremely expensive and labor-intensive to obtain \cite{li2021superpixel}. With limited budgets and efforts, the resulting dataset would be noisy, and the presence of label noises may mislead the segmentation model to memorize wrong semantic correlations, resulting in severely degraded generalizability \cite{karimi2020deep,zhou2021learning}. Hence, developing medical image segmentation techniques that are robust to noisy labels in training data is of great importance. 

\begin{wrapfigure}{r}{58mm}
\centering
\includegraphics[width=58mm]{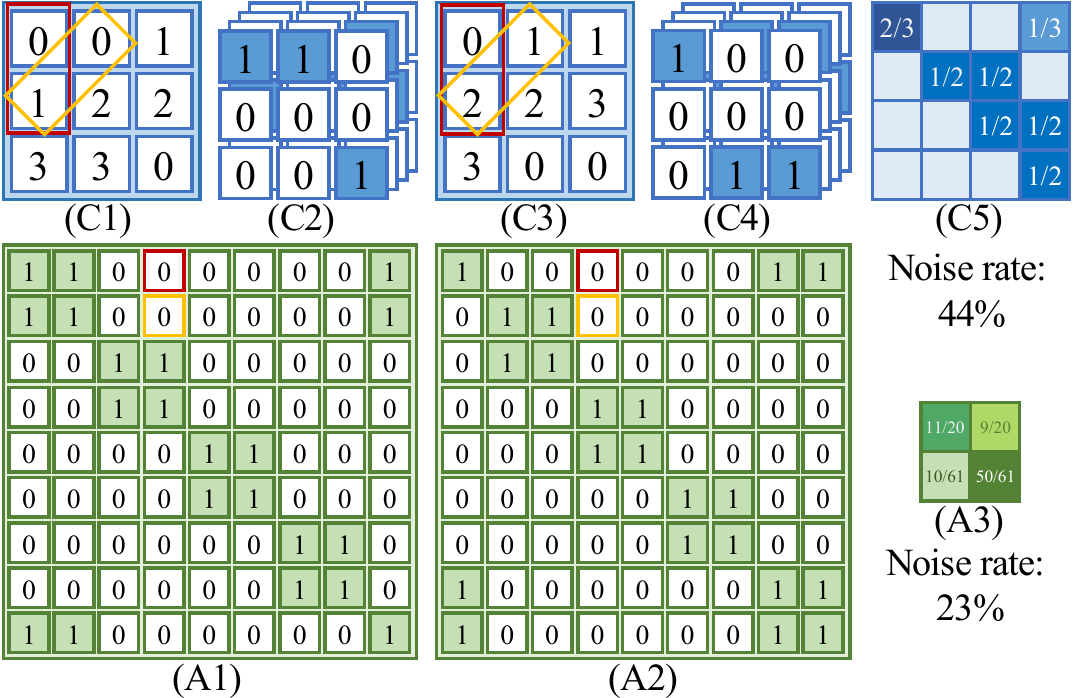}
\caption{A toy example to illustrate the comparison between pixel-wise class label (C) and pair-wise affinity label (A). (C1, C2) True class label and the one-hot encoding. (C3, C4) Noisy class label and the one-hot encoding. (C5) Class-level noise transition matrix with noise rate of 44\%. (A1) True affinity label. (A2) Noisy affinity label. (A3) Affinity-level noise transition matrix with noise rate of 23\%.}
\label{fig:Motivation}
\end{wrapfigure}


Solutions towards noisy label issues in image classification tasks have been extensively explored  \cite{li2021provably,shu2019meta,wu2021class2simi,zhang2020distilling,zhou2021learning}, while pixel-wise label noises in segmentation tasks have not been well-studied, especially for medical image analysis. Previous solutions for medical image segmentation with noisy labels can be summarized into three aspects. Firstly, some researchers model the noisy label distribution through either the confusion matrix \cite{zhang2020disentangling} or noise transition matrix (NTM) \cite{guo2022simt,guo2021metacorrection}, and then leverage the modeled distribution for {pixel-wise loss corrections}. Secondly, {pixel-wise label refurbishments} are implemented by the spatial label smoothing regularization \cite{xu2021noisy} or the convex combination with superpixel predictions \cite{li2021superpixel}. Thirdly, {pixel-wise resampling and reweighting strategies} are designed to concentrate the segmentation model on learning reliable pixels. For instance, Tri-network \textit{et al.} \cite{zhang2020robust} contains three collaborative networks and adaptively selects informative samples according to the consensus between predictions from different networks. Wang \textit{et al.} \cite{wang2020meta} leverage meta-learning to automatically estimate an importance map, thereby mining reliable information from important pixels. 

Despite the impressive performance in promoting generalizability, almost all existing image segmentation methods tackle label noise issues merely in a pixel-wise manner. 
\textit{Complementing the widely utilized pixel-wise manner, we make the first effort in exploiting the affinity relation between pixels within an image for noisy mitigation in a pair-wise manner}. Unlike {pixel-wise manner} that regularizes pixels with {class label} (Fig. \ref{fig:Motivation} C1-4), {pair-wise manner} constrains relations between pixels with {affinity label} (Fig. \ref{fig:Motivation} A1-2), indicating whether two pixels belong to the same category. The motivation behind this conception is to reduce the ratio of label noises. Intuitively, if one pixel in a pair is mislabeled (e.g. the \textcolor{red}{red} rectangle in Fig. \ref{fig:Motivation}) or even both pixels are mislabeled (e.g. the \textcolor{orange}{orange} rectangle in Fig. \ref{fig:Motivation}), the affinity label of this pair might be correct, thereby reducing the noise rate (e.g. from 44\% to 23\%  in Fig. \ref{fig:Motivation}). Moreover, affinity relations in pair-wise manner comprehensively incorporate intra-class and inter-class contextual dependencies, and thus it may be beneficial to explicitly differentiate them for correlated information propagation and irrelevant information elimination.

Unifying the pixel-wise and pair-wise manners, we propose a robust Joint Class-Affinity Segmentation (JCAS) framework to combat label noise issues in medical image segmentation. JCAS framework has two supervision signals, derived from noisy class labels and noisy affinity labels, for regularizing pixel-wise predictions and pair-wise affinity relations, respectively. {These two supervision signals in JCAS are complementary to each other since the pixel-wise one preserves semantics and the pair-wise one reduces noise rate.} Pair-wise affinity relations derived at the feature level model the contextual dependencies, indicating the correlation between any two pixels in a pair. Considering differentiated contextual dependencies can prevent undesirable aggregations, \textit{we devise a differentiated affinity reasoning (DAR) module to guide the refinement of pixel-wise predictions with differentiated affinity relations}. DAR module differentiates affinity relations to explicitly aggregate intra-class correlated information and eliminate inter-class irrelevant information. \textit{To further correct both pixel-wise and pair-wise supervision signals, we design a class-affinity loss correction (CALC) strategy.} This strategy models noise label distributions in class labels and affinity labels as two NTMs for loss correction, meanwhile, it unifies the pixel-wise and pair-wise supervisions through the theoretically derived consistency regularization, thereby facilitating the noise resistance. Extensive experiments under both synthetic and real-world noisy labels demonstrate the effectiveness of the proposed JCAS framework with a minimum gap towards the upper bound performance.

\section{Joint Class-Affinity Segmentation Framework}
\begin{figure}[tp!]
\centering
\includegraphics[width=122mm]{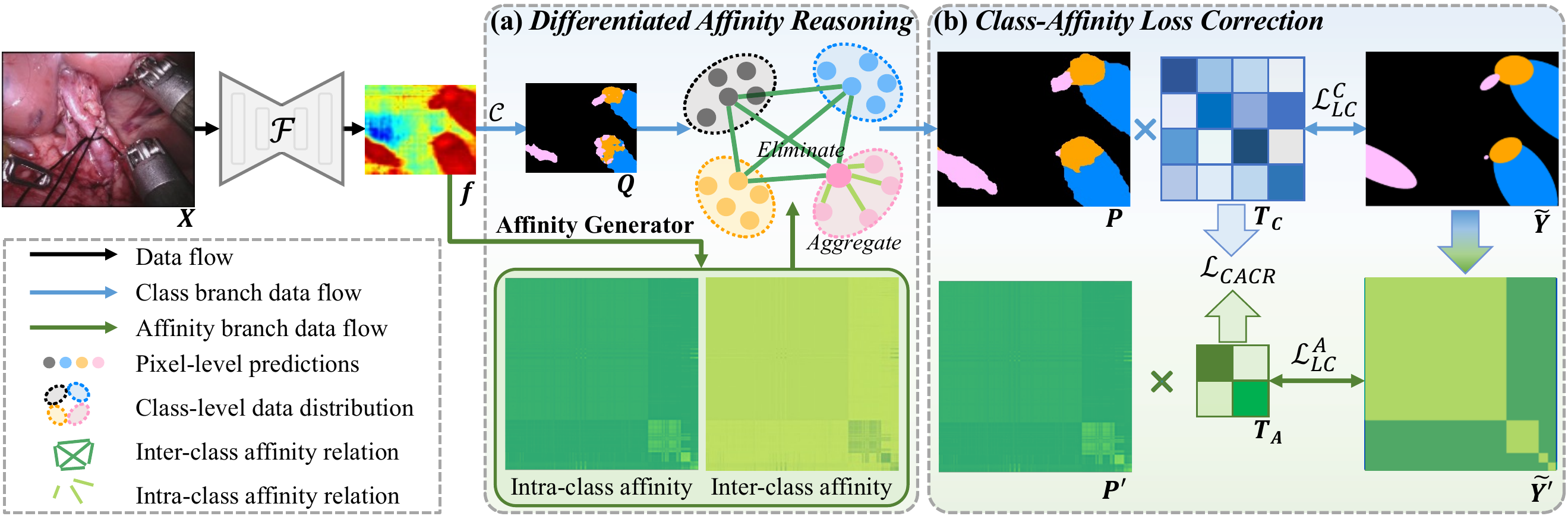}
\caption{Illustration of Joint Class-Affinity Segmentation (JCAS) framework, including (a) differentiated affinity reasoning and (b) class-affinity loss correction.}
\label{fig:Framework}
\end{figure}

The proposed Joint Class-Affinity Segmentation (JCAS) framework is illustrated in Fig. \ref{fig:Framework}. Formally, we have access to training images $\mathcal{X} = \{ \boldsymbol{X} \in \mathbb{R}^{H\times W\times 3} \}$ with spatial dimension of  $H\times W$. The corresponding one-hot encoding of pixel-wise noisy labels is denoted as $\mathcal{Y} = \{\widetilde{\boldsymbol{Y}} \in \mathbb{R}^{H\times W\times C} \}$, where $C$ indicates the number of classes. We aim to learn a segmentation network that is robust to label noises in $\mathcal{Y}$ during the training process and could derive clean labels for test data. Given an input training image $\boldsymbol{X}$, a feature map $\boldsymbol{f} \in \mathbb{R}^{h\times w\times d}$ is first computed from the feature extractor $\mathcal{F}$. Note that $h$, $w$, and $d$ denote the height, width, and channel number of the feature map. Then, the feature map is passed through two branches for estimating pixel-wise predictions (upper branch in Fig. \ref{fig:Framework}) and pair-wise affinity relations (lower branch in Fig. \ref{fig:Framework}), respectively. 

In the upper branch, a classifier $\mathcal{C}$ with softmax is used to produce the coarse segmentation result $\boldsymbol{Q}$. In the lower branch, an affinity generator is introduced to generate the affinity map $\boldsymbol{P'} \in [0, 1]^{n \times n}$ where $n$=$h \times w$, and the generator is formulated as $\boldsymbol{P'}(k_{1}, k_{2})=norm(\frac{\boldsymbol{f}(i_{1}, j_{1})^{\top} \boldsymbol{f}(i_{2}, j_{2})}{\left\|\boldsymbol{f}(i_{1}, j_{1})\right\|_{2}\left\|\boldsymbol{f}(i_{1}, j_{1})\right\|_{2}})$. $(i_{\cdot}, j_{\cdot})$ is the coordinate of a pixel in feature map, and $(k_{1}, k_{2})$ is the coordinate in affinity map. Note that $k_{1}$ and $(i_{1}, j_{1})$ denote the position of the same pixel. The operator $norm(\cdot)$ performs normalization along each row to ensure affinity relations towards pixel $k_{1}$ are summed to 1, i.e., $\sum_{k_{2}}^{n}\boldsymbol{P'}(k_{1}, k_{2}) = 1$. 
The obtained affinity map $\boldsymbol{P'}$ measures feature similarity between two pixels. Since intra-class pixels share the similar semantic features, intra-class pixel pairs usually show large similarity scores in $\boldsymbol{P'}$, which highlights these pixel pairs belonging to the same class. Hence, 
$\boldsymbol{P'}$ reveals the intra-class affinity relations. 
Then, we devise a differentiated affinity reasoning (DAR) module (Fig. \ref{fig:Framework} (a), Sec. \ref{DAR}) to obtain refined segmentation result $\boldsymbol{P}$, where the affinity map $\boldsymbol{P'}$ derived in the lower branch is leveraged to guide the refinement of previously generated coarse segmentation result $\boldsymbol{Q}$ in the upper branch. Both pixel-wise segmentation prediction $\boldsymbol{P}$ and pair-wise affinity map $\boldsymbol{P'}$ are regularized through the proposed class-affinity loss correction (CALC) strategy (Fig. \ref{fig:Framework} (b), Sec. \ref{CALC}). The optimized JCAS framework produces the refined segmentation result $\boldsymbol{P}$ as the final prediction in test phase.

\subsection{Differentiated Affinity Reasoning (DAR)}\label{DAR}
In the image segmentation task, each image is equipped with a ground truth map, indicating pixel-wise semantic class label. Pixel-wise supervision signal cannot regularize the segmentation network to model the contextual dependencies from isolated pixels. Hence, we incorporate the contextual dependency embedded in the pair-wise affinity map $\boldsymbol{P'}$ to guide the refinement of the pixel-wise segmentation result $\boldsymbol{Q}$. Moreover, different from existing affinity methods \cite{xu2021leveraging,zhu2021weakly} that aggregates contextual information as a mixture and may introduce undesirable contextual aggregations, we propose a differentiated affinity reasoning (DAR) module to explicitly distinguish intra-class and inter-class contextual dependencies and leverage the differentiated contexts to rectify segmentation predictions. 

In addition to previously calculated pair-wise affinity map $\boldsymbol{P'}$ that represents intra-class affinity relation, we infer the reverse affinity map $\boldsymbol{P'_{re}}=norm(1-\boldsymbol{P'})$. The reverse affinity map measures the dissimilarity between two pixels and reveals the inter-class affinity relations. The proposed DAR module performs intra-class and inter-class affinity reasonings, respectively. To be specific, the intra-class affinity reasoning aims to aggregate correlated information according to the intra-class affinity relations $\boldsymbol{P'}$, and the inter-class affinity reasoning aims to eliminate irrelevant information according to the inter-class affinity relations $\boldsymbol{P'}_{re}$, which can be formulated as:
\begin{equation}
\begin{footnotesize}
\begin{aligned}
\begin{split}
\boldsymbol{P}_{intra}(k_{1})=\boldsymbol{P}(k_{1}) + \sum_{k_{2}}^{n} \boldsymbol{P'}(k_{1}, k_{2})  \boldsymbol{Q}(k_{2});\boldsymbol{P}_{inter}(k_{1})=\boldsymbol{P}(k_{1}) - \sum_{k_{2}}^{n} \boldsymbol{P'_{re}}(k_{1}, k_{2})  \boldsymbol{Q}(k_{2}).
\end{split}
\end{aligned}
\end{footnotesize}
\label{eqDAR}
\end{equation}The refined pixel-wise prediction $\boldsymbol{P}$ is obtained through combining both intra-class and inter-class affinity reasoning results, i.e., $\boldsymbol{P}=\frac{1}{2}(\boldsymbol{P}_{intra}+\boldsymbol{P}_{inter})$. With the proposed DAR module, the correct predictions are strengthened, and the incorrect segmentation results are debiased and rectified. 

\subsection{Class-Affinity Loss Correction (CALC)}\label{CALC}
In multi-class image segmentation task, the widely used cross entropy loss is computed in a pixel-wise manner and formulated as $\mathcal{L}^{C}_{CE} = - \sum_{k}^{H\times W}\widetilde{\boldsymbol{Y}}(k) \log \boldsymbol{P}(k)$. However, directly minimizing the empirical risk of training data with respect to noisy labels $\widetilde{\boldsymbol{Y}}$ will lead to severely degraded generalizability. To reduce the noise rate, we introduce the pair-wise manner, and the corresponding affinity label is derived by $\boldsymbol{Y'}(k1, k2)=\boldsymbol{Y}(k1)^{\top}\boldsymbol{Y}(k2)$. Only if two pixels share the same category, the value in the affinity label $\boldsymbol{Y'}$ will be 1, otherwise $\boldsymbol{Y'}$ will be 0. Although the pair-wise manner can greatly reduce the noise rate compared to the pixel-wise manner as demonstrated in Fig. \ref{fig:Motivation}, there still exist noises, and thus the binary entropy loss $\mathcal{L}^{A}_{Bi} = - \sum_{k}^{H\times W}\widetilde{\boldsymbol{Y}}(k) \log \boldsymbol{P'}(k) + (1-\widetilde{\boldsymbol{Y}}(k)) \log (1-\boldsymbol{P'}(k))$ for affinity map supervision cannot guarantee the robustness of segmentation model towards label noises, resulting in biased semantic correlations. To facilitate the noise tolerance of $L^{C}_{CE}$ and $L^{A}_{Bi}$, we devise the class-affinity loss correction (CALC) strategy, including the class-level loss correction $\mathcal{L}^{C}_{LC}$ and affinity-level loss correction $\mathcal{L}^{A}_{LC}$. Meanwhile, a theoretically derived class-affinity consistency regularization $\mathcal{L}_{CACR}$ is advanced to unify pixel-wise and pair-wise supervisions.

\textbf{\textit{Class-level Loss Correction.}} We model the label noise distributions in noisy class labels through a noise transition matrix (NTM) $\boldsymbol{T}_{C}\in [0,1]^{C\times C}$, which specifies the probability of clean label $m$ translating to noisy label $n$ via $\boldsymbol{T}_{C}(m, n)=p(\widetilde{Y}=n|Y=m)$. Hence, the probability of one pixel being predicted as $\widetilde{Y}=n$ is computed by $p(\widetilde{Y}=n)=\sum_{m=1}^{C} p(Y=m)\cdot \boldsymbol{T}_{C}(m, n)$, where $p(Y)$ is the clean class probability. 
Then the modeled noise label distribution is exploited to correct the supervision signal (i.e. $L^{C}_{CE}$) derived from noisy labels via $\mathcal{L}^{C}_{LC}=- \sum_{k}^{H\times W}\widetilde{\boldsymbol{Y}}(k) \log [\boldsymbol{P}(k)\boldsymbol{T}_{C}]$. This corrected loss encourages the consistency between noisy translated predictions and noisy class labels. Therefore, once the true NTM is obtained, the desired estimation of clean class predictions can be recovered by the output of segmentation model $\boldsymbol{P}$. For the estimation of the true NTM, we exploit the volume minimization regularization in \cite{li2021provably}.

\textbf{\textit{Affinity-level Loss Correction.}} Similar to the class-level NTM, affinity-level NTM is defined as $\boldsymbol{T}_{A}\in [0, 1]^{2\times 2}$, modeling the probability of clean affinity labels flipping to noisy affinity labels. 
Then, we exploit the modeled label noise distribution NTM to rectify the supervision signal (i.e. $L^{A}_{Bi}$) for affinity relation learning. Therefore, the affinity-level loss correction is formulated as $\mathcal{L}^{A}_{LC} = - \sum_{k}^{H\times W}\widetilde{\boldsymbol{Y}}(k) \log [\boldsymbol{P'}(k)\boldsymbol{T}_{A}] + (1-\widetilde{\boldsymbol{Y}}(k)) \log (1-\boldsymbol{P'}(k)\boldsymbol{T}_{A})$.

\textbf{\textit{Class-Affinity Consistency Regularization.}} To unify the pixel-wise and pair-wise supervisions, we bridge the class-level and affinity-level NTMs in Theorem \ref{Theorem}. A theoretical proof for the Theorem is provided in Sec. 5 \textit{supplementary}. Hence, the class-affinity consistency regularization is defined as $\mathcal{L}_{CACR}=\left\| \boldsymbol{T}_{C\to A} - \boldsymbol{T}_{A}\right\|_{2}$. 

Combining the above defined losses, we obtain the joint loss of the proposed JCAS framwork as: $\mathcal{L}=\mathcal{L}^{C}_{LC} + \mathcal{L}^{A}_{LC} + \lambda \mathcal{L}_{CACR}$, which interacts the pixel-wise and pair-wise manners. Note that $\lambda$ is the weighting factor of $\mathcal{L}_{CACR}$. 
\begin{theorem}\label{Theorem}
Assume that the class distribution of dataset denoting proportions of pixel number is $\mathcal{N}=[N_{1}, N_{2}, ..., N_{C}]$, and the noise is class-dependent
\footnote{Real-world label noises can be well approximated via class-dependent noises \cite{cheng2020learning,guo2022simt,li2021provably}.}. 
Given a class-level NTM $\boldsymbol{T}_{C}$, the translated affinity-level NTM $\boldsymbol{T}_{C\to A}$ is calculated by
\begin{equation}
\begin{footnotesize}
\begin{array}{l}
\boldsymbol{T}_{C\to A}(0, 0)=1-\boldsymbol{T}_{C\to A}(0, 1), ~\boldsymbol{T}_{C\to A}(0, 1)=\frac{\sum_{m} \left[N_{m}\sum_{n}\boldsymbol{T}_{C}(m, n)\right]^{2} - \sum_{m}(N_{m})^{2} \left \| \boldsymbol{T}_{C} \right \|_{2}^{2}}{\sum_{m}\left[ N_{m}(\sum_{m}N_{m}-N_{m})\right]},\\
\boldsymbol{T}_{C\to A}(1,0)=1-\boldsymbol{T}_{C\to A}(1, 1),  ~\boldsymbol{T}_{C\to A}(1, 1)=\frac{\sum_{m}(N_{m})^{2} \left \| \boldsymbol{T}_{C} \right \|_{2}^{2}}{\sum_{m}(N_{m})^{2}}.
\end{array}
\end{footnotesize}
\end{equation}
\end{theorem}

\section{Experiments}

\begin{figure}[tp!]
\centering
\includegraphics[width=122mm]{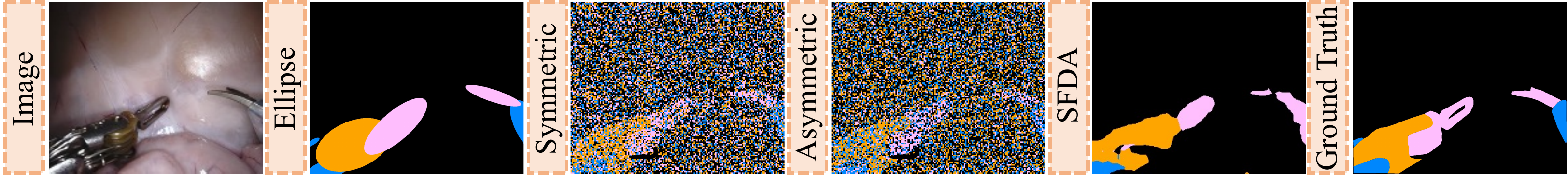}
\caption{Illustration of dataset with different kinds of label noises.}
\label{fig:dataset}
\end{figure}

\textbf{Dataset.} We validate the proposed JCAS method on the surgical instrument dataset Endovis18 \cite{allan20202018}. It consists of 2384 images annotated with the instrument part labels, and the label space includes \textcolor{blue}{shaft}, \textcolor{orange}{wrist} and \textcolor{pink}{clasper} classes, as shown in Fig. \ref{fig:dataset}. The dataset is split into 1639 training images and 596 test images following \cite{gonzalez2020isinet}. 
Each image is resized into a resolution of 256$\times$320 in preprocessing. 

\noindent\textbf{Noise Patterns.} To comprehensively verify the robustness of JCAS, we conduct experiments with both \textbf{\textit{synthetic label noise}} (i.e., \textit{elipse}, \textit{symmetric} and \textit{asymmetric} noises) and \textbf{\textit{real-world label noise}} (i.e., \textit{noisy pseudo labels in source-free domain adaptation (SFDA)}), as compared in Fig. \ref{fig:dataset}. Specifically, the ellipse noisy label is a kind of weak annotation generated by drawing the minimal ellipse given the true segmentation label, greatly reducing the manual annotation cost. To simulate errors in the annotation process, ellipse labels are randomly dilated and eroded. Moreover, two commonly used label noises in the machine learning field, including symmetric and asymmetric noises with the rate of 0.5 \cite{li2021provably,zhou2021learning}, are used to evaluate JCAS. Furthermore, we introduce Endovis17 \cite{allan20192017} containing 1800 annotated images with domain shift to Endovis18, and generate realistic noisy labels from source only model trained on Endovis17. 

\noindent\textbf{Implementation.} 
The proposed JCAS framework is implemented with PyTorch on Nvidia 2080Ti. DeepLabV2 \cite{chen2017deeplab} with the pre-trained encoder ResNet101 is our segmentation backbone. The initial learning rate is set as 1e-4 for the pre-trained encoder and 1e-3 for the rest of trainable parameters. We adopt a batch size of 3 and the maximum epoch number of 200. The weighting factor $\lambda$ is 0.01. The segmentation performance is assessed by $Dice$ and $Jac$ scores.

\begin{table}[t!]
\centering
\caption{Comparison under four label noises. Best and second best results are \textbf{highlighted} and \underline{underlined}. `w/ Affinity' introduces pair-wise supervision $\mathcal{L}^{A}_{Bi}$ to backbone.}
\label{table:results}
\scalebox{0.8}{\begin{tabular}{c  |  c  |  c  c | c  c | c  c | c  c}
\toprule[1.2pt]
\multirow{2}*{Noises}&\multirow{2}*{Method}&\multicolumn{2}{c}{Shaft}&\multicolumn{2}{c}{Wrist}&\multicolumn{2}{c}{Clasper}&\multicolumn{2}{c}{Average}\\
\cline{3-10}
&&\textit{Dice (\%)}&\textit{Jac (\%)}&\textit{Dice (\%)}&\textit{Jac (\%)}&\textit{Dice (\%)}&\textit{Jac (\%)}&\textit{Dice (\%)}&\textit{Jac (\%)}\\
\hline

&Upper bound&88.740&81.699&65.045&52.627&70.531&56.618&74.772&63.648 \\
\hline

&RAUNet (19$'$) \cite{ni2019raunet}&83.137&74.139&56.941&43.215&61.081&45.883&67.053&54.412 \\
 
     
&LWANet (20$'$) \cite{ni2020attention}&81.945&72.735&53.626&40.886&\textbf{64.364}&\textbf{49.781}&66.645&54.468 \\

\multirow{7}*{Ellipse}&CSS (21$'$) \cite{pissas2021effective}&\underline{84.577}&\textbf{75.736}&57.597&43.687&63.686&48.347&\underline{68.620}&\underline{55.923} \\

&MTCL (21$'$) \cite{xu2021noisy}&72.719&60.540&39.386&27.474&49.662&35.085&53.922&41.033 \\

&SR (21$'$) \cite{zhou2021learning}&79.966&69.621&53.540&39.747&60.179&44.775&64.561&51.381 \\

&VolMin (21$'$) \cite{li2021provably}&81.320&70.758&60.470&46.408&58.203&42.524&66.664&53.230 \\
     
&Baseline \cite{chen2017deeplab}&79.021&68.097&42.069&29.582&55.489&40.175&58.860&45.951 \\

&w/ Affinity&82.158&72.339&49.128&35.455&58.933&43.594&63.406&50.463 \\

&w/ DAR&82.698&72.992&52.207&38.442&61.544&46.027&65.483&52.487 \\

&w/ CALC&82.973&73.126&\underline{61.885}&\underline{47.527}&60.416&44.821&68.425&55.158 \\


&Ours (JCAS)&\textbf{84.683}&\underline{75.378}&\textbf{65.599}&\textbf{51.623}&\underline{63.871}&\underline{48.356}&\textbf{71.384}&\textbf{58.452} \\
     
\bottomrule[0.6pt]
&RAUNet (19$'$) \cite{ni2019raunet}&68.044&54.397&31.581&20.676&41.302&27.819&46.976&34.297 \\
      
&LWANet (20$'$) \cite{ni2020attention}&0.294&0.150&10.089&5.908&10.228&5.489&6.870&3.849 \\

\multirow{4}*{Symmetric}&CSS (21$'$) \cite{pissas2021effective}&86.555&78.451&32.363&20.767&53.364&37.901&57.427&45.706 \\

&MTCL (21$'$) \cite{xu2021noisy}&78.480&67.855&50.011&38.013&55.515&40.411&61.336&48.760 \\
     
&SR (21$'$) \cite{zhou2021learning}&86.648&78.823&58.217&46.870&64.643&50.120&69.836&58.604 \\

&VolMin (21$'$) \cite{li2021provably}&\underline{86.811}&\underline{78.834}&\underline{63.712}&\underline{51.259}&\underline{66.604}&\underline{52.096}&\underline{72.376}&\underline{60.730} \\
  
&Baseline \cite{chen2017deeplab}&85.021&76.419&57.026&44.563&63.255&48.395&68.434&56.459 \\     

&Ours (JCAS)&\textbf{88.285}&\textbf{80.692}&\textbf{65.759}&\textbf{53.487}&\textbf{68.129}&\textbf{53.821}&\textbf{74.058}&\textbf{62.667} \\
     
\bottomrule[0.6pt]
&RAUNet (19$'$) \cite{ni2019raunet}&87.255&79.983&59.462&46.639&67.347&52.801&71.355&59.808 \\

&LWANet (20$'$) \cite{ni2020attention}&0.015&0.007&40.548&30.683&9.060&4.825&16.541&11.838 \\

\multirow{4}*{Asymmetric}&CSS (21$'$) \cite{pissas2021effective}&\textbf{89.825}&\textbf{83.543}&43.743&30.569&\textbf{69.285}&\textbf{54.758}&67.618&56.290 \\

&MTCL (21$'$) \cite{xu2021noisy}&74.544&62.525&41.433&30.533&48.077&33.676&54.685&42.244 \\
     
&SR (21$'$) \cite{zhou2021learning}&86.360&78.055&62.854&49.651&65.483&50.962&71.566&59.556 \\
          
&VolMin (21$'$) \cite{li2021provably}&86.840&78.796&\underline{63.345}&\underline{51.137}&65.220&50.996&\underline{71.802}&\underline{60.310} \\

&Baseline \cite{chen2017deeplab}&84.497&75.607&58.717&46.060&61.662&46.770&68.292&56.146 \\

&Ours (JCAS)&\underline{88.247}&\underline{80.730}&\textbf{67.298}&\textbf{54.922}&\underline{67.686}&\underline{53.436}&\textbf{74.410}&\textbf{63.029} \\
     
\bottomrule[0.6pt]
&RAUNet (19$'$) \cite{ni2019raunet}&73.370&61.568&56.063&42.570&45.979&31.720&58.471&45.286 \\
 
&LWANet (20$'$) \cite{ni2020attention}&75.377&64.457&53.203&39.799&\underline{48.558}&\underline{34.191}&59.046&46.149 \\

\multirow{4}*{SFDA}&CSS (21$'$) \cite{pissas2021effective}&74.419&64.261&\textbf{61.765}&\textbf{47.880}&45.749&31.709&\underline{60.644}&\underline{47.950} \\

&MTCL (21$'$) \cite{xu2021noisy}&72.289&60.346&51.095&37.972&38.762&25.567&54.048&41.295 \\
  
&SR (21$'$) \cite{zhou2021learning}&75.992&64.835&57.370&43.863&40.471&27.388&57.944&45.362 \\

&VolMin (21$'$) \cite{li2021provably}&\textbf{76.641}&\underline{65.063}&58.285&44.389&41.780&28.324&58.902&45.925 \\
   
&Baseline \cite{chen2017deeplab}&76.107&64.858&56.259&42.740&41.364&28.091&57.910&45.230 \\
     
&Ours (JCAS)&\underline{76.540}&\textbf{65.300}&\underline{59.904}&\underline{46.104}&\textbf{48.725}&\textbf{34.283}&\textbf{61.723}&\textbf{48.562} \\
\bottomrule[1.2pt]
\end{tabular}}
\end{table}

\noindent\textbf{Experiment results.} 
Experimental comparison results under four types of label noises are presented in Table \ref{table:results}, in which we list the performance of upper bound (i.e., model trained with clean labels), three state-of-the-arts \cite{ni2019raunet,ni2020attention,pissas2021effective} in instrument segmentation, three label noise methods \cite{xu2021noisy,zhou2021learning,li2021provably}, our backbone \cite{chen2017deeplab}, and the proposed JCAS. For a fair comparison, we reimplement \cite{xu2021noisy,zhou2021learning,li2021provably} using the same backbone \cite{chen2017deeplab}. Compared with the aforementioned baselines, JCAS shows the minimum performance gap with the upper bound under all kinds of label noises, demonstrating the robustness of JCAS. Despite the satisfactory performance under ellipse and SFDA noises, LWANet \cite{ni2020attention} cannot deal with the other two types of noises, resulting in 6.870\% and 16.541\% $Dice$ scores. In contrast, JCAS shows comparable result to the upper bound with only 0.981\% and 0.388\% $Jac$ gaps under symmetric and asymmetric label noises. We further illustrate typical surgical instrument segmentation results in Fig. \ref{fig:results}, validating the superiority of JCAS over baseline methods in the qualitative aspect.

\begin{figure}[t!]
\centering
\includegraphics[width=122mm]{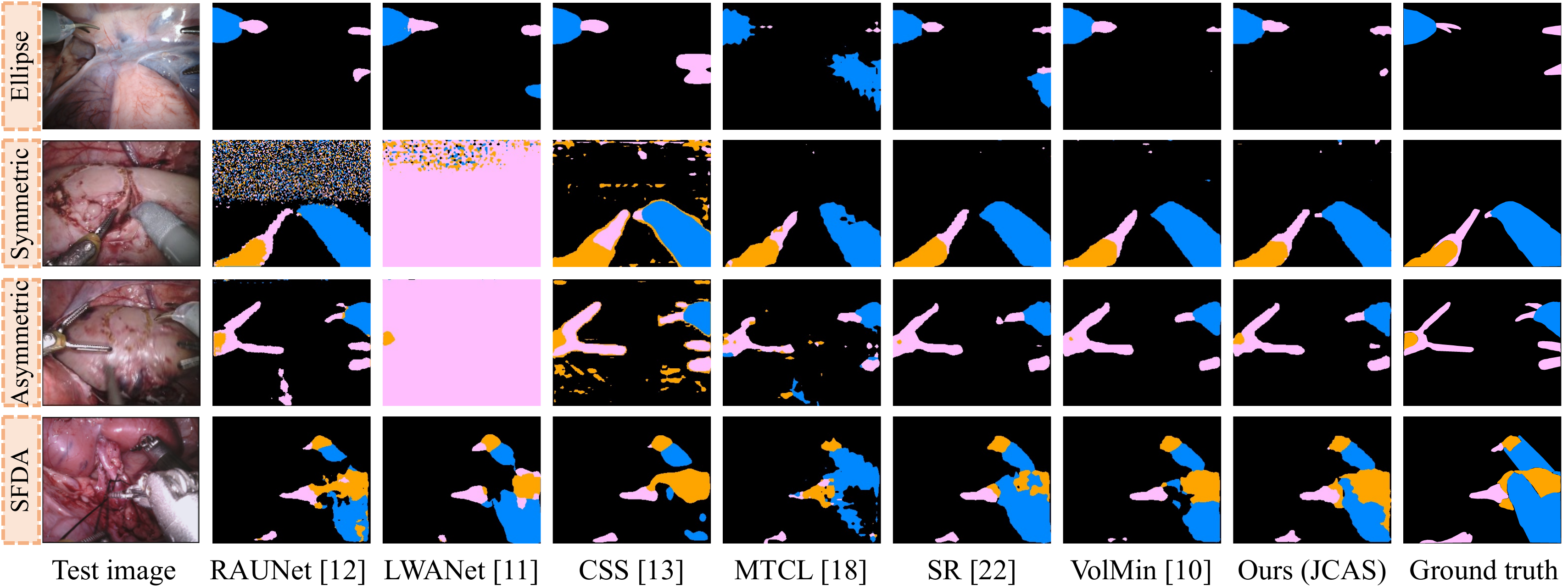}
\caption{Comparison of segmentation results.}
\label{fig:results}
\end{figure}

\begin{figure}[t!]
\centering
\includegraphics[width=122mm]{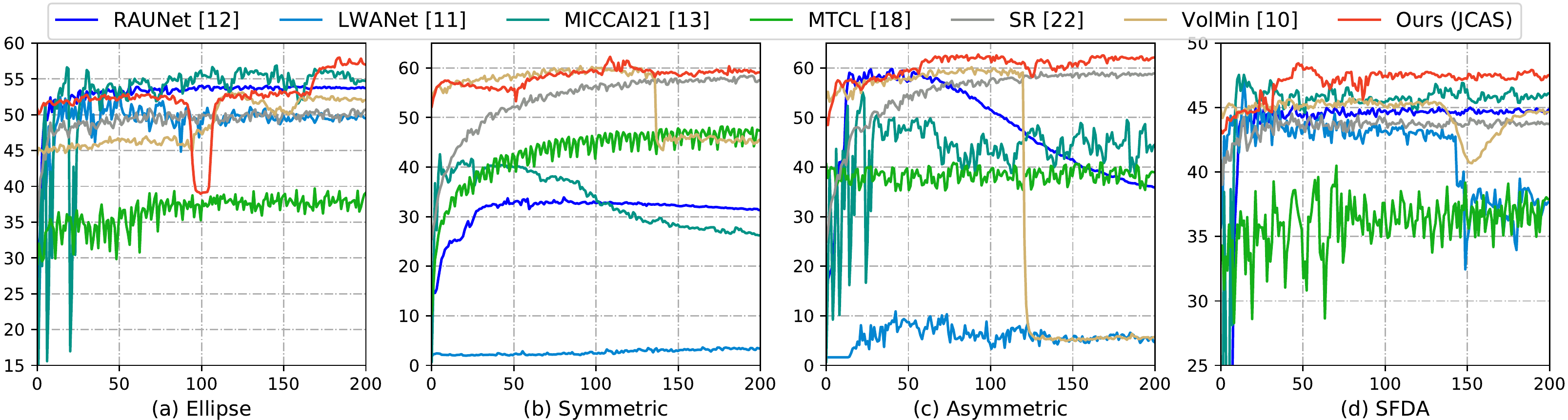}
\caption{Curve of test $Jac$ vs. epoch with four different types of noise labels.}
\label{fig:curve}
\end{figure}

To analyze the influence of JCAS, we then conduct ablation study under ellipse noises. With the pair-wise manner (`w/ Affinity'), the noise rate of supervision signals is greatly reduced, yielding an increment of 4.512\% in $Jac$, while the increased memory overhead is negligible (from 780.71MB to 784.67MB). The devised DAR module (`w/ DAR') is also verified to be effective in differentiating contexture dependencies for the refinement of segmentation predictions, achieving an improvement of 9.207\% $Jac$ score compared to the backbone. Moreover, the proposed CALC strategy further rectifies supervision signals derived from noisy labels and boosts the segmentation performance with 5.965\% $Jac$ gain. 
To verify each component in DAR and CALC, we further ablate intra-class affinity reasoning ($1^{st}$ item in Eq. (\ref{eqDAR})), inter-class affinity reasoning ($2^{nd}$ item in Eq. (\ref{eqDAR})), $\mathcal{L}_{LC}^{C}$, $\mathcal{L}_{LC}^{A}$, and $\mathcal{L}_{CACR}$ under ellipse noises, obtaining 55.795\%, 56.180\%, 55.203\%, 55.179\%, 56.612\% $Jac$. The performance of ablating each component is degraded compared to 58.452\% $Jac$ achieved by our method (Table \ref{table:results}), verifying the effectiveness of individual component in mitigating label noise issue.

Furthermore, we show test $Jac$ curves in Fig. \ref{fig:curve}. While \cite{li2021provably,ni2019raunet} obtain promise results under ellipse and SFDA noises, they reach a high $Jac$ in the early stage and then decrease, overfitting to the other two noises. Notably, our JCAS converges to high performance under four noises and demonstrates more stable training process compared to \cite{ni2020attention,pissas2021effective,xu2021noisy}, verifying its noise-resistant property.

\section{Conclusion}
In this paper, we propose a robust JCAS framework to combat label noise issues in medical image segmentation. Complementing the widely used pixel-wise manner, we introduce the pair-wise manner by capturing affinity relations among pixels to reduce noise rate. Then a DAR module is devised to rectify pixel-wise segmentation predictions by reasoning about intra-class and inter-class affinity relations. We further design a CALC strategy to unify pixel-wise and pair-wise supervisions, and facilitate noise tolerances of both supervisions. Extensive experiments under four noisy labels corroborate the noise immunity of JCAS.

\section{Supplementary}

\subsection{Proof of Theorem \ref{Theorem}}
\begin{theorem}\label{Theorem}
Assume that the class distribution of dataset denoting proportions of pixel number is $\mathcal{N}=[N_{1}, N_{2}, ..., N_{C}]$, and the noise is class-dependent. Given a class-level NTM $\boldsymbol{T}_{C}$, the translated affinity-level NTM $\boldsymbol{T}_{C\to A}$ is calculated by
\begin{equation}
\begin{footnotesize}
\begin{array}{l}
\boldsymbol{T}_{C\to A}(0, 0)=1-\boldsymbol{T}_{C\to A}(0, 1), ~~\boldsymbol{T}_{C\to A}(0, 1)=\frac{\sum_{m} \left[N_{m}\sum_{n}\boldsymbol{T}_{C}(m, n)\right]^{2} - \sum_{m}(N_{m})^{2} \left \| \boldsymbol{T}_{C} \right \|_{2}^{2}}{\sum_{m}\left[ N_{m}(\sum_{m}N_{m}-N_{m})\right]}, \\
\boldsymbol{T}_{C\to A}(1,0)=1-\boldsymbol{T}_{C\to A}(1, 1),  ~~\boldsymbol{T}_{C\to A}(1, 1)=\frac{\sum_{m}(N_{m})^{2} \left \| \boldsymbol{T}_{C} \right \|_{2}^{2}}{\sum_{m}(N_{m})^{2}}.
\end{array}
\end{footnotesize}
\end{equation}
\end{theorem}

\begin{proof}
Noise transition matrix (NTM) $\boldsymbol{T}_{C}\in [0,1]^{C\times C}$ specifies the probability of clean label $Y=m$ translating to noisy label $\widetilde{Y}=n$, which can be formulated as $\boldsymbol{T}_{C}(m, n)=p(\widetilde{Y}=n|Y=m)$.  Taking the entry $\boldsymbol{T}_{C\to A}(0, 0)$ of affinity-level NTM for example, we first calculate the number of pixel pairs with clean affinity labels $Y'=0$ through $\sum_{m\ne m'}N_{m}N_{m'}\boldsymbol{T}_{C}(m, n)\boldsymbol{T}_{C}(m', n')$, and compute the number of data pairs with clean affinity labels $Y'=0$ and noisy affinity labels $\widetilde{Y'}=0$ via $\sum_{m\ne m', n\ne n'}N_{m}N_{m'}\boldsymbol{T}_{C}(m, n)\boldsymbol{T}_{C}(m', n')$. Hence, the proportion of these two terms derives the element $\boldsymbol{T}_{C\to A}(0, 0)=p(\widetilde{Y'}=0 | Y'=0)=\frac{\sum_{m\ne m', n\ne n'}N_{m}N_{m'}\boldsymbol{T}_{C}(m, n)\boldsymbol{T}_{C}(m', n')}{\sum_{m\ne m'}N_{m}N_{m'}\boldsymbol{T}_{C}(m, n)\boldsymbol{T}_{C}(m', n')}$. Similar to the derivation of entry $\boldsymbol{T}_{C\to A}(0, 0)$, we can obtain the remaining three entries, and thus we have:
\begin{equation}
\boldsymbol{T}_{C\to A}(0, 0)=\frac{\sum_{m\ne m', n\ne n'}N_{m}N_{m'}\boldsymbol{T}_{C}(m, n)\boldsymbol{T}_{C}(m', n')}{\sum_{m\ne m'}N_{m}N_{m'}\boldsymbol{T}_{C}(m, n)\boldsymbol{T}_{C}(m', n')},
\end{equation}
\begin{equation}
\boldsymbol{T}_{C\to A}(0, 1)=\frac{\sum_{m\ne m', n= n'}N_{m}N_{m'}\boldsymbol{T}_{C}(m, n)\boldsymbol{T}_{C}(m', n')}{\sum_{m\ne m'}N_{m}N_{m'}\boldsymbol{T}_{C}(m, n)\boldsymbol{T}_{C}(m', n')},
\end{equation}
\begin{equation}
\boldsymbol{T}_{C\to A}(1, 0)=\frac{\sum_{m= m', n\ne n'}N_{m}N_{m'}\boldsymbol{T}_{C}(m, n)\boldsymbol{T}_{C}(m', n')}{\sum_{m= m'}N_{m}N_{m'}\boldsymbol{T}_{C}(m, n)\boldsymbol{T}_{C}(m', n')},
\end{equation}
\begin{equation}
\boldsymbol{T}_{C\to A}(1, 1)=\frac{\sum_{m= m', n= n'}N_{m}N_{m'}\boldsymbol{T}_{C}(m, n)\boldsymbol{T}_{C}(m', n')}{\sum_{m= m'}N_{m}N_{m'}\boldsymbol{T}_{C}(m, n)\boldsymbol{T}_{C}(m', n')}.
\end{equation}

Further, note that
\begin{equation}
\begin{aligned}
\begin{split}
\sum_{m\ne m'}&N_{m}N_{m'}\boldsymbol{T}_{C}(m, n)\boldsymbol{T}_{C}(m', n') \\
&=\sum_{m\ne m'}(N_{m}\sum_{n}\boldsymbol{T}_{C}(m, n))(N_{m'}\sum_{n'}\boldsymbol{T}_{C}(m', n'))=\sum_{m\ne m'}N_{m}N_{m'}\\
&=\sum_{m}\left[N_{m}\sum_{m\ne m'}N_{m'}\right]=\sum_{m}\left[ N_{m}(\sum_{m}N_{m}-N_{m})\right] \\
\end{split}
\end{aligned}
\end{equation}
\begin{equation}
\begin{aligned}
\begin{split}
\sum_{m= m'}&N_{m}N_{m'}\boldsymbol{T}_{C}(m, n)\boldsymbol{T}_{C}(m', n') \\
&=\sum_{m= m'}(N_{m}\sum_{n}\boldsymbol{T}_{C}(m, n))(N_{m'}\sum_{n'}\boldsymbol{T}_{C}(m', n'))=\sum_{m= m'}N_{m}N_{m'}\\
&=\sum_{m}(N_{m})^{2} \\
\end{split}
\end{aligned}
\end{equation}
\begin{equation}
\begin{aligned}
\begin{split}
\sum_{m= m', n=n'}&N_{m}N_{m'}\boldsymbol{T}_{C}(m, n)\boldsymbol{T}_{C}(m', n')\\
=&\sum_{m, n}\left[N_{m} \boldsymbol{T}_{C}(m, n)\right]^{2}=\sum_{m}(N_{m})^{2} \sum_{n} \left[\boldsymbol{T}_{C}(m, n)\right]^{2}=\sum_{m}(N_{m})^{2} \left \| \boldsymbol{T}_{C} \right \|_{2}^{2} \\
\end{split}
\end{aligned}
\end{equation}
\begin{equation}
\begin{aligned}
\begin{split}
\sum_{m\ne m', n=n'}&N_{m}N_{m'}\boldsymbol{T}_{C}(m, n)\boldsymbol{T}_{C}(m', n) ~~~~~~~~~~~~~~~~~~~~~~~~~~~~~~~~~~~~~~~~~~~~~\\
=&\sum_{m\ne m'}(N_{m}\sum_{n}\boldsymbol{T}_{C}(m, n))(N_{m'}\sum_{n}\boldsymbol{T}_{C}(m', n))\\
=&\sum_{m} \left[N_{m}\sum_{n}\boldsymbol{T}_{C}(m, n)\right]^{2}- \sum_{m}(N_{m})^{2} \left \| \boldsymbol{T}_{C} \right \|_{2}^{2}\\
\end{split}
\end{aligned}
\end{equation}
Substituting the derived equations above to Eq. (4-5), we have proved the Theorem \label{Theorem}. 
\end{proof}

\subsection{Implementation of Class-Affinity Consistency Regularization}

Given class-level noise transition matrix ($Tc$), affinity-level noise transition matrix ($Ta$), and class distribution ($N$), the proposed class-affinity consistency regularization can be derived through the code shown in Listing \ref{lstlisting:CACR_loss}, and the completed code will be published online. Since the true class distribution is not available due to the noisy labels, we leverage the class distribution of pseudo labels generated from the warm-up model as an approximation. 

\begin{lstlisting}[language=Python, caption=Implementation of class-affinity consistency regularization., label={lstlisting:CACR_loss}]
import torch

def CACR_loss(Tc, Ta, N):
    v00 = v01 = v10 = v11 = 0
    num_classes = Tc.shape[0]
    for m1 in range(num_classes):
        for n1 in range(num_classes):
            a = t[m1][n1]
            
            for m2 in range(num_classes):
                for n2 in range(num_classes):
                    b = t[m2][n2]
                    
                    if m1 == m2 and n1 == n2:
                        v11 += a * b * N[m1] * N[m2]
                    if m1 == m2 and n1 != n2:
                        v10 += a * b * N[m1] * N[m2]
                    if m1 != m2 and n1 == n2:
                        v01 += a * b * N[m1] * N[m2]
                    if m1 != m2 and n1 != n2:
                        v00 += a * b  * N[m1] * N[m2]
                        
    Tc_a = torch.zeros([2, 2]).cuda()
    Tc_a[0][0] = v11 / (v11 + v10)
    Tc_a[0][1] = v10 / (v11 + v10)
    Tc_a[1][0] = v01 / (v01 + v00)
    Tc_a[1][1] = v00 / (v01 + v00)
    loss = torch.nn.MSELoss(reduction='mean')(Tc_a, Ta)
    return loss
    
\end{lstlisting}

\subsection{Visualization Results}
We provide more surgical instrument segmentation results under four label noises, including ellipse (Fig. \ref{fig:results_ellipse}), symmetric (Fig. \ref{fig:results_sym}), asymmetric (Fig. \ref{fig:results_asym}) and SFDA (Fig. \ref{fig:results_sfda}) noises. The qualitative comparison results demonstrate the superiority of the proposed JCAS framework in learning precise semantic correlations.

\begin{figure}[h!]
\vspace{-0.2cm}
\centering
\includegraphics[width=122mm]{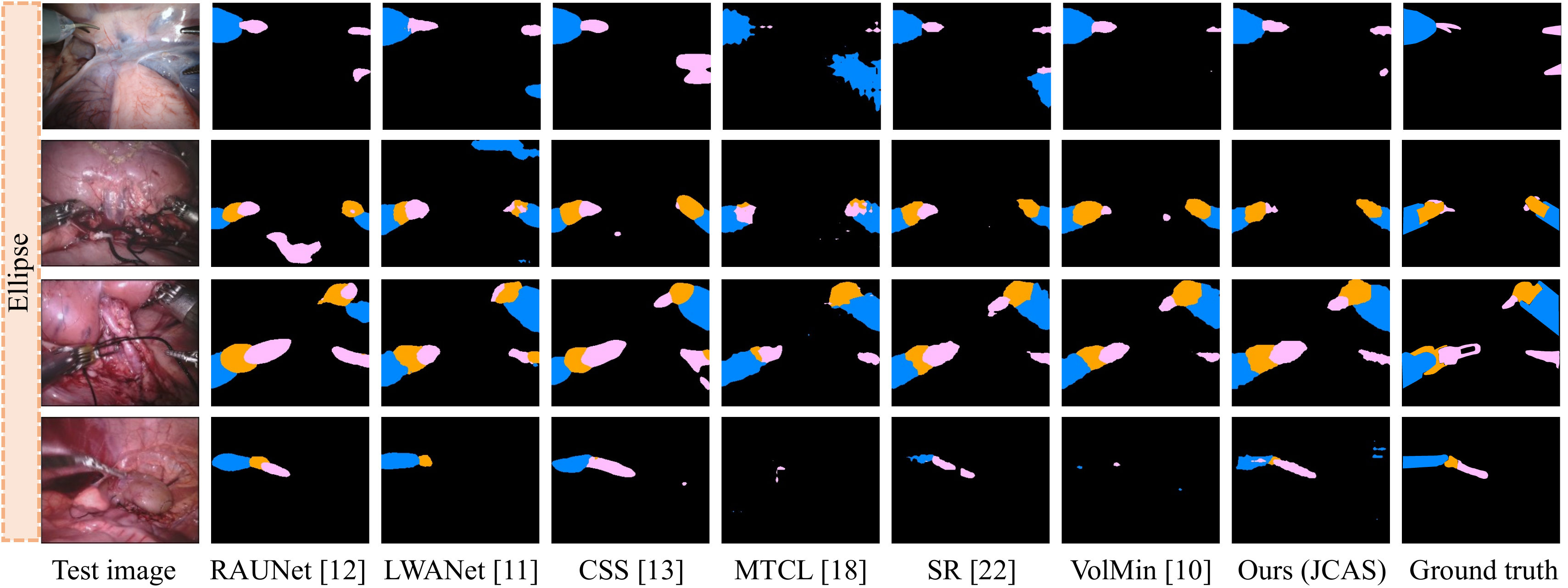}
\caption{Illustration of dataset with ellipse label noises.}
\label{fig:results_ellipse}
\vspace{-0.2cm}
\end{figure}

\begin{figure}[h!]
\vspace{-0.4cm}
\centering
\includegraphics[width=122mm]{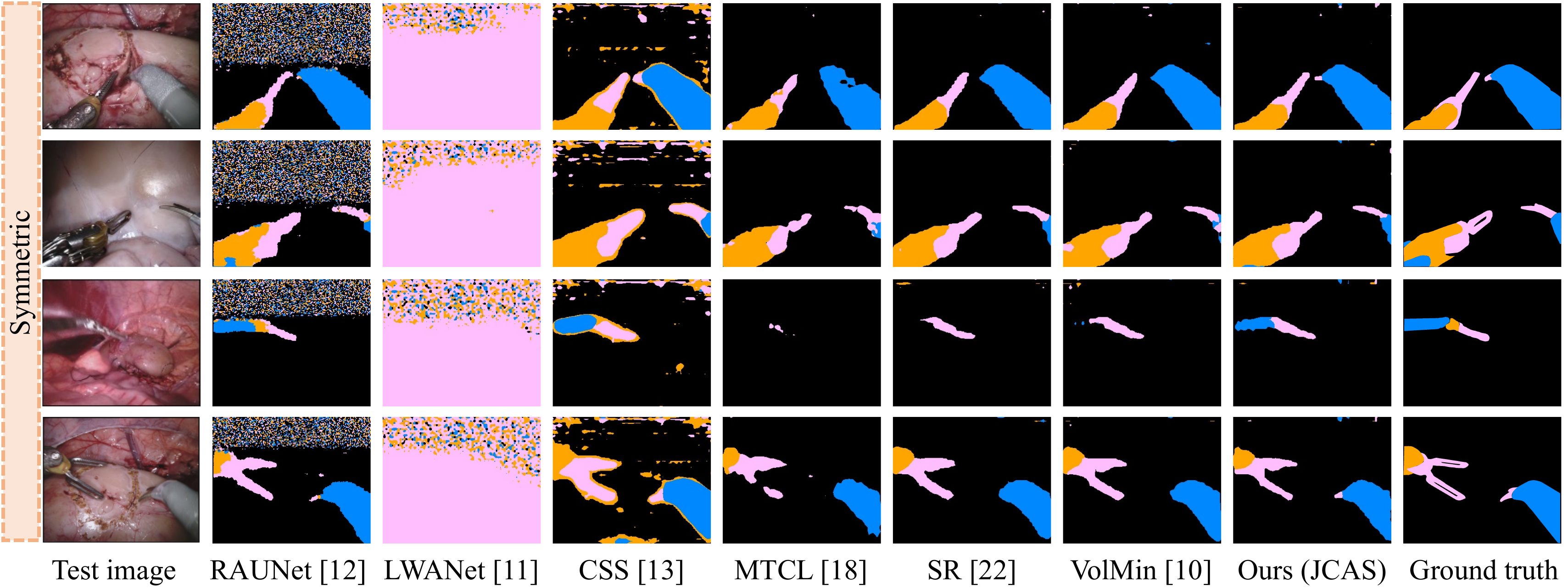}
\caption{Illustration of dataset with symmetric label noises.}
\label{fig:results_sym}
\end{figure}

\begin{figure}[h!]
\centering
\includegraphics[width=122mm]{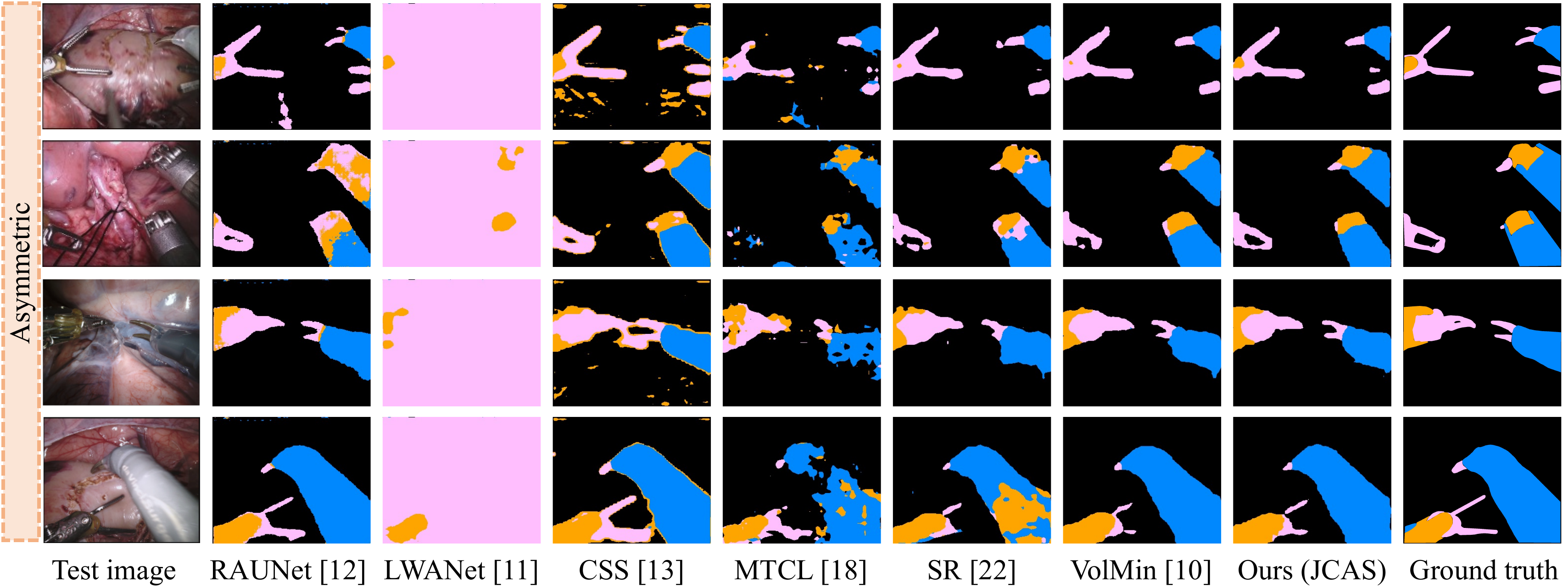}
\caption{Illustration of dataset with asymmetric label noises.}
\label{fig:results_asym}
\end{figure}

\begin{figure}[h!]
\centering
\includegraphics[width=122mm]{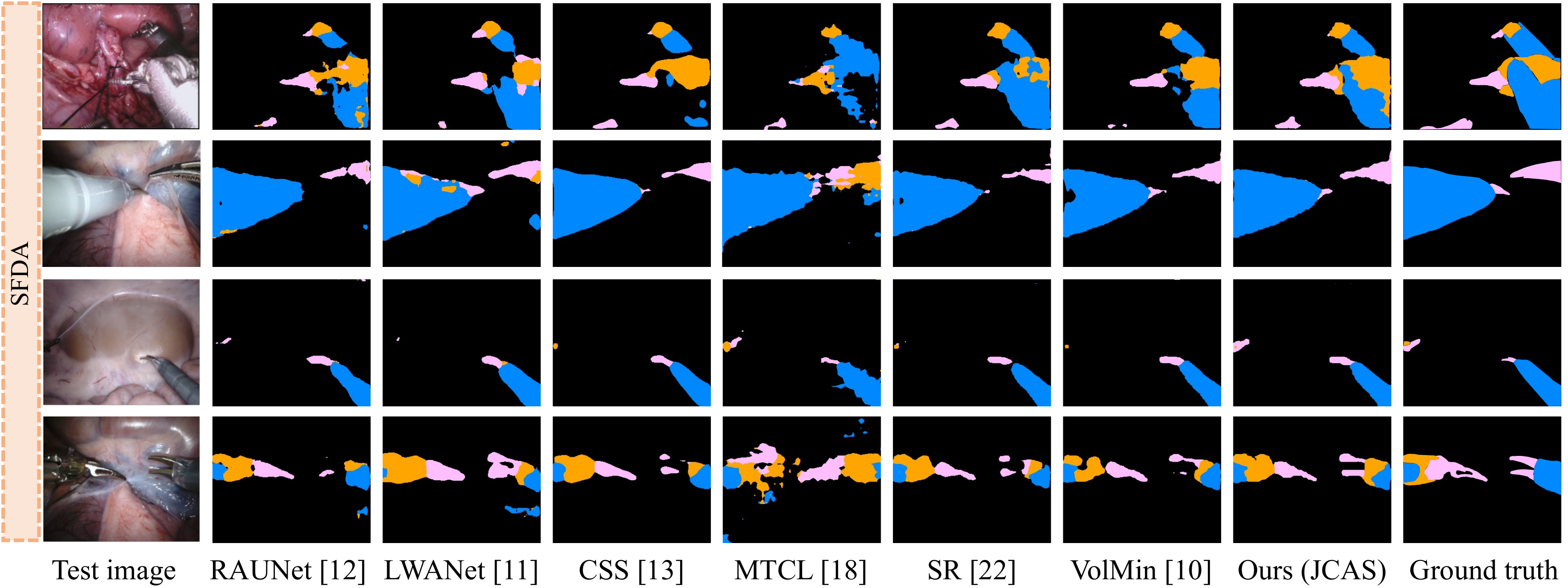}
\caption{Illustration of dataset with SFDA label noises.}
\label{fig:results_sfda}
\end{figure}

\subsubsection{Acknowledgments.} This work was supported by Hong Kong Research Grants Council (RGC) Early Career Scheme grant 21207420 (CityU 9048179) and Hong Kong Research Grants Council (RGC) General Research Fund 11211221(CityU 9043152).

%
%
%
%


\end{document}